\documentclass{arxiv}

\usepackage[T1]{fontenc}
\usepackage{textcomp}

\usepackage{amsmath,amsthm,amssymb}
\usepackage{booktabs}                   %
\usepackage[detect-all]{siunitx}        %

\usepackage[dvipsnames]{xcolor}         %
\usepackage{graphicx}
\usepackage{hyperref}                   %

\usepackage{wrapfig}                    %
\usepackage{multirow}                   %
\usepackage{caption}                    %
\usepackage{listings}                   %
\usepackage{algorithm}                  %
\usepackage{algpseudocode}              %
\algtext*{EndFor}                       %
\algtext*{EndIf}                        %
\usepackage[most]{tcolorbox}            %
\usepackage[square,numbers,sort&compress]{natbib}

\usepackage[capitalize,noabbrev]{cleveref}

\hypersetup{
   pdftitle={Feedback Descent: Open-Ended Text Optimization via Pairwise Comparison},
   urlcolor=NavyBlue,
   linkcolor=NavyBlue,
   citecolor=ForestGreen!70!black
}

\newtheorem{proposition}{Proposition}

\colorlet{my-red}{BrickRed!90!Sepia}
\colorlet{my-blue}{Aquamarine!30!Blue}

\newtcblisting{promptbox}[1][Prompt Template for Feedback Descent]{
    colback=gray!5!white,
    colframe=gray!75!black,
    listing only,
    title={#1},
    fonttitle=\bfseries,
    boxrule=0.8pt,
    left=2mm,
    enhanced,
    breakable,
    listing options={
        basicstyle=\ttfamily\small,
        breaklines=true,
        columns=fullflexible,
        keepspaces=true,
        breakindent=0pt,
        breakautoindent=false,
    }
}
\tcbset{
  takeawaybox/.style={
    width=0.9\linewidth,
    top=6pt,
    bottom=0pt,
    left=8pt,
    right=8pt,
    colback=my-blue!3!white,
    colframe=my-blue,
    colbacktitle=my-blue,
    enhanced,
    center,
    attach boxed title to top left={yshift=-0.1in,xshift=0.15in},
    boxed title style={boxrule=0pt,colframe=white,},
  }
}
\newtcolorbox{takeawaybox}[2][]{takeawaybox,title=#2,#1}

\usepackage{amsmath,amsfonts,bm}

\def\eqref#1{equation~\ref{#1}}

\def\1{\bm{1}}

\newcommand{\blueul}[1]{\color{blue}{\underline{#1}}}

\title{Feedback Descent: Open-Ended Text Optimization via Pairwise Comparison}

\author[1]{Yoonho Lee}
\author[1]{Joseph Boen}
\author[1]{Chelsea Finn}

\affil[1]{Stanford University}

\correspondingauthor{yoonho@cs.stanford.edu}

\begin{document}

\begin{abstract}
We introduce \textit{Feedback Descent}, a framework that optimizes text artifacts---prompts, code, and molecules---through structured textual feedback rather than relying solely on scalar rewards. By preserving detailed critiques instead of compressing them to binary preferences, Feedback Descent widens the information bottleneck in preference learning, enabling directed optimization in text space rather than weight space. We show that in-context learning can transform structured feedback into gradient-like directional information, enabling targeted edits. Unlike prior approaches that collapse judgments into single bits, our evaluators pair each comparison with textual feedback, which functions as high-bandwidth supervision. The iteration loop is done purely at inference time, without modifying any model weights, and is task-agnostic. We evaluate Feedback Descent on three diverse domains and find that it outperforms state-of-the-art prompt optimization (GEPA), reinforcement learning methods (GRPO, REINVENT), and even specialized graph-based molecular optimizers. In the DOCKSTRING molecule discovery benchmark, Feedback Descent identifies novel drug-like molecules surpassing the $99.9$th percentile of a database with more than $260{,}000$ compounds across six protein targets.
\end{abstract}

\maketitle

\section{Introduction}

A central goal of machine learning is building systems that can perform tasks beyond human capabilities. Reinforcement learning is a powerful framework that accomplishes this goal, since it can optimize with respect to feedback on its own outputs, rather than relying on supervised examples of desired outputs. Indeed, recent language models have demonstrated impressive feats in domains like math and programming~\citep{openai2024o1,deepseek2025r1,alphageometry2025,zhu2024deepseek} through a combination of reinforcement learning and text-based reasoning. Unfortunately, existing reinforcement learning frameworks are designed to learn from impoverished supervision signals, typically either scalar rewards or pairwise preference data, where each annotation conveys at most a single bit per pair. These bottlenecks discard information about \textit{why} one behavior is better and \textit{how} to improve---information available in environment feedback or easily elicited from humans during annotation~\citep{wu2023finegrained,just2024datacentric}.

Our goal is to widen this information bottleneck, i.e., significantly increase the information the system can extract per unit of experience~\citep{silver2025era}.
Collecting more detailed feedback is straightforward, e.g., with brief rationales explaining preferences; the challenge is turning such feedback into measurable improvement.
Because free-form feedback does not define a differentiable objective, it cannot directly drive weight updates via backpropagation.
Our approach iterates at inference time, using language models to translate accumulated feedback into targeted edits of text artifacts (prompts, code, molecules, JSON configs)
that improve a final performance objective, without any weight updates.

To that end, we introduce \textit{Feedback Descent},
a framework for continual optimization in text space.
At each iteration, we prompt a language model to propose an improved version of the current best artifact, conditioned on all previous feedback. We compare this candidate against the current best, and the evaluator returns a preference along with textual feedback explaining the choice.
If the candidate is preferred, it becomes the new best.
Repeating this loop yields semantically local, feedback-aligned improvements that implement gradient-like steps in text space.
See~\cref{fig:pull_figure} for a conceptual illustration.
We provide theoretical intuition for why Feedback Descent can be effective. Under appropriate assumptions about feedback quality and problem structure, we demonstrate that textual feedback can provide directional information, enabling efficient optimization.

Our contributions are threefold.
\textbf{First}, we formalize why textual feedback enables dimension-free convergence while zeroth-order methods suffer exponential slowdown with effective dimensionality, identifying when and why structured feedback outperforms scalar rewards.
\textbf{Second}, we demonstrate cross-domain generality: Feedback Descent works across three qualitatively distinct domains (visual design, prompt optimization, molecule design) with the same iterative loop.
\textbf{Third}, we validate competitive or superior performance versus specialized methods, achieving competitive results with the state-of-the-art in prompt optimization (GEPA) while outperforming a reinforcement learning baseline (GRPO). In the molecule design experiment, Feedback Descent outperforms specialized molecular optimizers (Graph MCTS/GA, REINVENT) despite operating purely on text representations, discovering molecules that surpass the 99.9th percentile of a 260,000-compound database of molecules.

\begin{figure}[t]
\centering \includegraphics[width=0.9\linewidth]{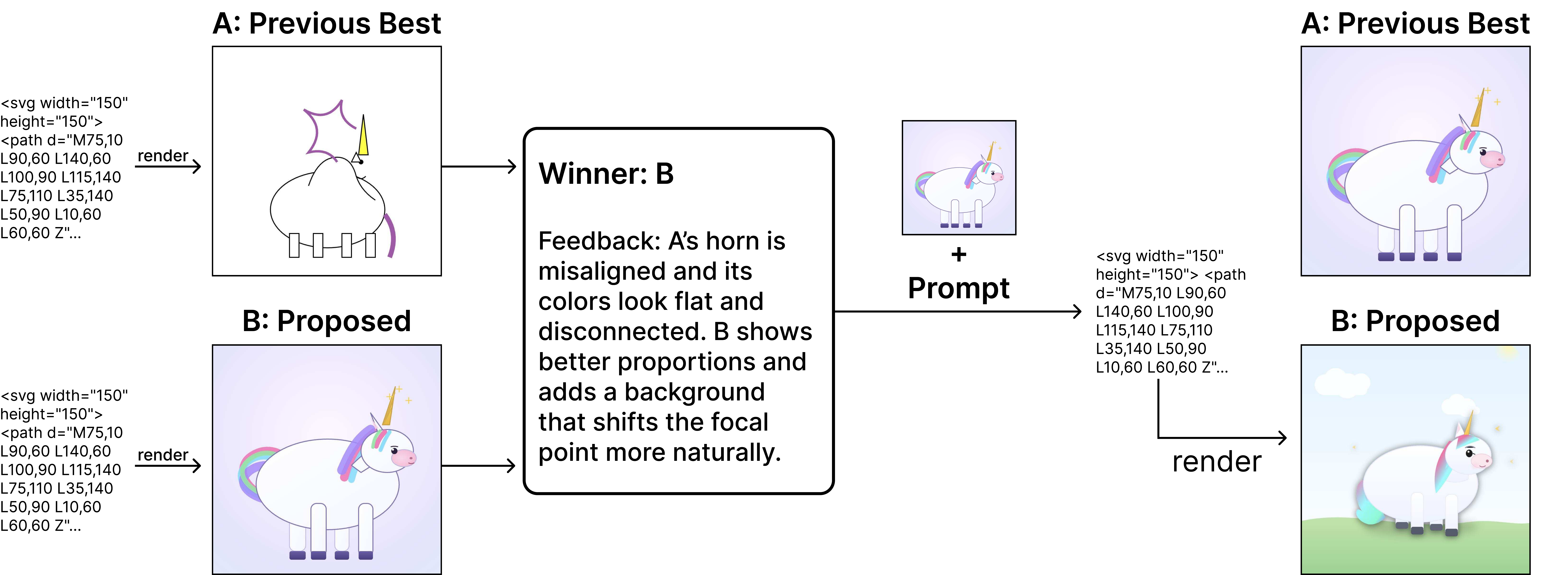}
\caption{A conceptual illustration of feedback descent. At each iteration, we compare the previous best artifact with a new candidate. The evaluator provides both a pairwise preference and textual feedback. Preferences ensure the selection of better candidates, while feedback accumulates directional information that guides semantically meaningful edits.}
\label{fig:pull_figure}
\end{figure}

\section{Feedback Descent: Open-Ended Text Optimization}

We propose Feedback Descent, a framework for open-ended optimization of text-representable artifacts whose quality is easier to \emph{judge} than to \emph{construct}. Feedback Descent converts comparative textual feedback into directed semantic edits and iterates in a self-improvement loop.
As a running example, consider optimizing SVG code to render better images of a unicorn. Current vision-language models can reliably compare two renderings and explain the choice, even if writing high-quality SVG from scratch is difficult. Through Feedback Descent, we can convert these explanations into directed edits that aim to produce an artifact that surpasses all previous ones.

\subsection{Problem Setup}
Let $\mathcal{S}$ be the space of token sequences, and let $x\in\mathcal{S}$ denote an artifact (e.g., SVG code). Given the current best $x_t^\star\in\mathcal{S}$ and a candidate $x\in\mathcal{S}$, the evaluator returns
\begin{equation}
    \label{eq:evaluator}
\mathsf{E}(x, x_t^\star) \to \big(p\in\{0,1\},\ r\in\mathcal{S}\big),
\end{equation}
where $p=1$ indicates $x \succ x_t^\star$ and $r$ is a textual feedback explaining \emph{why} the winner is better and \emph{how} to improve. We append $r_t$ to a history $\mathcal{R}_t=\{(x_1, r_1),\ldots,(x_t, r_t)\}$ and iterate, keeping track of the current best artifact $x_t^\star$.

\subsection{Feedback Descent}

Feedback Descent operates as an iterative optimization loop that maintains a single best artifact $x_t^*$ and progressively improves it through feedback-guided mutations and comparative evaluation. Throughout, we use $\mathcal{M}$ to denote the language model used for generating improved candidates.

\textbf{Initialization and termination.}
We initialize $x_0^*$ by prompting a language model with the task description alone (e.g., "Generate SVG code for a unicorn"). The algorithm runs for a fixed budget of $T$ iterations or until convergence (defined as no improvement for $k$ consecutive iterations).

\textbf{Proposing semantic mutations via prompting.}
The mutation step leverages a language model's in-context learning capabilities. At iteration $t$, we prompt the model with the current best artifact $x_t^*$ and accumulated feedback $\mathcal{R}_{t-1}$ to generate an improved candidate:
\begin{equation}
    \label{eq:mutation}
x_t = \mathcal{M}\left(x_t^*, \mathcal{R}_{t-1}\right)
\end{equation}
The prompt instructs the model to address previous critiques while preserving successful elements. These prompts are intentionally minimal: the optimization signal comes from the accumulated feedback rather than heavy prompt engineering.
\begin{wrapfigure}{r}{0.48\linewidth}
\vspace{-7mm}
\begin{minipage}{\linewidth}
\begin{algorithm}[H]
\caption{\label{alg:feedback_descent} Feedback Descent}
\begin{algorithmic}[1]
\Require Initial text $x_0$, Language model $\mathcal{M}$, $T$
\State $\textrm{Current best: } x^* \leftarrow x_0, \textrm{Rationale history: } \mathcal{R} \leftarrow \emptyset$
\For{$t = 1$ \textbf{to} $T$}
    \State $x_t \leftarrow \mathcal{M}(x^*, \mathcal{R})$ \Comment{Propose (\ref{eq:mutation})}
    \State $p_t, r_t \leftarrow \textsc{Compare}(x_t, x^*)$ \Comment{Compare (\ref{eq:evaluator})}
    \State $\mathcal{R} \leftarrow \mathcal{R} \cup \{(x_t, r_t)\}$
    \If{$p_t = 1$}
        \State $x^* \leftarrow x_t, \mathcal{R} \leftarrow \emptyset$ \Comment{Update + reset}
    \EndIf
\EndFor
\State \Return $x^*$
\end{algorithmic}
\end{algorithm}
\vspace{-7mm}
\end{minipage}
\end{wrapfigure}
They include basic task context, the current best artifact, and feedback from previous comparisons. Complete prompt templates for each domain are provided in~\cref{app:prompts}.

\textbf{Selection and update.}
We compare the new candidate $x_t$ against the current best $x_t^*$ using the evaluator $\mathsf{E}(x_t, x_t^*)$, which returns both a binary preference $p_t$ and a textual feedback $r_t$.
In our running SVG example, examples of feedback include ``adjust the stroke width'', ``make sure the legs are connected to the body'', and ``add a shadow to the unicorn's mane''. Regardless of the preference outcome, we always add the feedback to our history: $\mathcal{R}_{t+1} = \mathcal{R}_t \cup \{(x_t, r_t)\}$. If $p_t = 1$ (candidate is preferred), we update $x_{t+1}^* = x_t$; otherwise we keep $x_{t+1}^* = x_t^*$. We summarize the overall process in \cref{alg:feedback_descent}.

\subsection{Analogy to Gradient Descent}
The key algorithmic insight is best understood by analogy to the gradient descent algorithm. Just as gradients provide the direction of steepest ascent under local linearity, textual feedback can suggest plausible directions of improvement in semantic space. For our SVG example, if the feedback indicates ``needs more defined horn shape,'' we expect that a small edit to the horn shape that preserves overall structure will likely be an improvement.

Of course, textual feedback is not a literal gradient. It is approximate and occasionally contradictory---optimization with such feedback does not have convergence guarantees in the same way that gradient descent does. Instead, feedback acts as a heuristic directional cue, offering higher-bandwidth supervision than a binary preference signal or a scalar reward, just as first-order optimization is fundamentally faster than zeroth-order optimization~\citep{nemirovski1983problem,agarwal2011information,nesterov2017random}. We hypothesize that an open-ended optimization loop based on such cues can succeed, supported by prior evidence that language models reliably translate textual instructions into concrete modifications. Examples include generating code changes~\citep{chen2021evaluating, austin2021program, nijkamp2022codegen, wang2023codet5+, roziere2023code, guo2024deepseek, lozhkov2024starcoder2, codegemma2024}, following complex multi-step instructions~\citep{ouyang2022training, wei2022finetuned, chung2022scaling, longpre2023flan, zhang2024instruction}, targeted text modifications~\citep{schick2022peer, du2022read, madaan2023selfrefine, welleck2023generating, mishra2023help}, and decomposing high-level goals into executable action sequences~\citep{schick2023toolformer, parisi2022talm, yao2022react, qin2023tool,wang2024llm,agarwal2025toolrm}.

\textbf{Why directional information helps.}
Zeroth-order methods that rely only on function evaluations or binary preferences suffer severe dimension-dependent slowdowns: convergence rates degrade exponentially as the search space grows~\citep{nemirovski1983problem,nesterov2017random}.
In contrast, first-order methods exploit gradient information to achieve dimension-free convergence under standard assumptions.
Textual feedback provides an approximation to such directional information.
Even when individual rationales are imperfect, their aggregate message across failures continually refines the direction of improvement.
We formalize this intuition in~\cref{app:proof}, showing that under idealized assumptions, rationale-guided updates can achieve linear convergence rates independent of effective dimensionality, while zeroth-order baselines scale exponentially worse.
These results provide motivation rather than rigorous guarantees for the discrete text domains we study empirically.
In~\cref{sec:experiments}, we show that Feedback Descent indeed produces consistent improvements across tasks, validating that such heuristic directional cues are sufficient to drive open-ended text optimization.

\section{Related Work}

\textbf{Preference Learning.}
Preference learning methods learn from pairwise comparisons~\citep{christiano2017deep,ouyang2022training,azar2023general,ethayarajh2024kto,munos2024nash}; recent advances include bypassing the need for a reward model~\citep{rafailov2023direct}, iterative optimization under KL constraints~\citep{xiong2023iterative}, and adaptive scaling techniques~\citep{wang2024adaptive}.
However, these methods fundamentally compress complex human reasoning into binary or scalar preferences, thereby forgoing the rich explanatory content that humans can naturally provide alongside their judgments~\citep{wirth2017survey}. Recent work shows that fine-grained feedback significantly improves reward modeling~\citep{wu2023finegrained,yu2024selfgenerated}, and incorporating rationales alongside preferences provides richer training signals~\citep{just2024datacentric}. Unlike prior work that uses rationales to improve weight-based training, we leverage detailed textual feedback entirely at inference time, widening the information bottleneck without requiring retraining.

\textbf{Evolutionary Algorithms and Gradient-Free Optimization.}
Feedback Descent can be viewed as an evolutionary algorithm~\citep{golberg1989genetic,holland1992adaptation}, in which candidates are iteratively mutated and accepted based on fitness.
While the black-box nature of modern LLMs has spurred interest in applying gradient-free approaches~\citep{guo2023match,sun2024bbtv2,chen2024derivative,lange2024evolution}, these methods face fundamental challenges in high-dimensional spaces.
More broadly, zeroth-order methods~\citep{chen2019zoadamm} face convergence rates that scale poorly with dimension, which is consistent with our experimental results comparing with reinforcement learning methods in~\cref{sec:experiments}.
To our best knowledge,~\citet{wang2024efficient} was the first prior work to demonstrate that LLMs can be effective molecule optimizers inside evolutionary algorithm loops.
Feedback Descent explores whether textual rationales can provide useful directional information for optimization, similar to how~\citet{nie2024importance} shows that LLMs can be effective optimizers when provided with directional feedback from historical traces.
Our contribution is in operationalizing an effective \textit{directed mutation operator} via accumulated textual feedback.

\textbf{Optimizing Compound AI Systems.}
Compound AI systems, i.e., modular architectures involving multiple LLM invocations and complex control flow, such as agents or scaffolding techniques~\citep{yao2022react}, present unique optimization challenges due to their modularity. Several approaches have emerged to tackle this complexity, including optimization for searching and bootstrapping few-shot in-context examples~\citep{khattab2022dspy,khattab2024dspy,opsahl2024optimizing} and reflective prompt evolution~\citep{agrawal2025gepa}.
The closest prior work to Feedback Descent is TextGrad~\citep{yuksekgonul2024textgrad}, which proposes PyTorch-like framework for automatic differentiation via text, backpropagating textual gradients through computation graphs. A core difference is that TextGrad optimizes ``pointwise'', i.e., it proposes a new instance of the artifact based only on the latest one. In contrast, Feedback Descent keeps an explicit trajectory-level buffer of comparative feedback. As we will see in~\cref{sec:experiments}, Feedback Descent is much more scalable at long-horizon optimization compared to TextGrad.

\begin{figure}[t]
\centering
\includegraphics[width=0.9\linewidth]{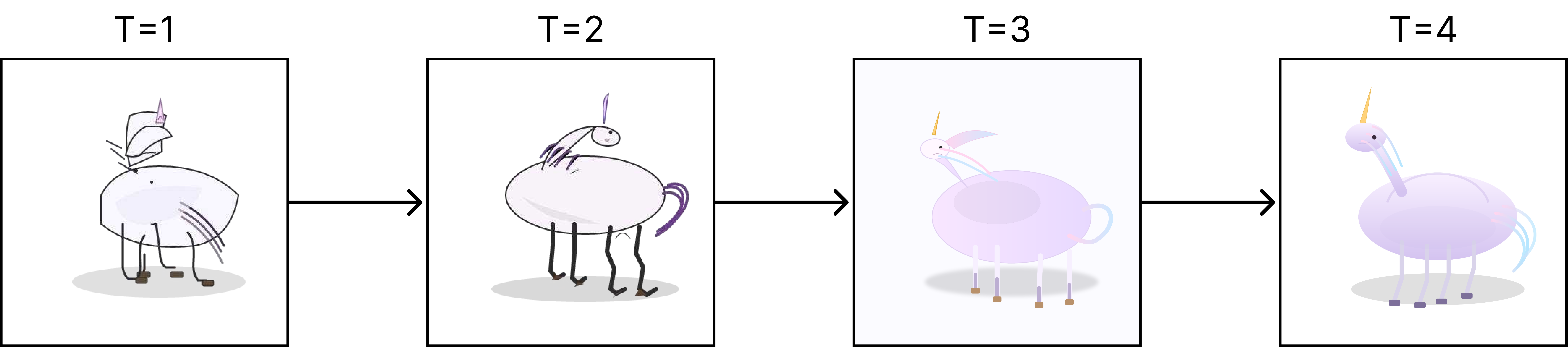}
\caption{Iterative progression of SVG unicorn optimization under the realism judge. \textbf{Feedback Descent produces gradual, semantically meaningful improvements through accumulating directional cues.}}
\label{fig:visual_progression}
\end{figure}

\textbf{Inference-Time Optimization for LLMs.}
Inference-time optimization improves performance without weight updates by performing additional computation at generation. This paradigm includes self-critique and refinement cycles (constitution-guided critique~\citep{bai2022constitutional}; Self-Refine~\citep{madaan2023selfrefine}) test-time scaling via best-of-$N$, multi-step reasoning, and tree search~\citep{cobbe2021training,zelikman2022star,yao2023tree}, and iterative prompt optimization~\citep{zhou2022ape,yang2023opro,pryzant2023protegi}. Several works report that strategically allocating inference-time compute yields large gains~\citep{snell2024scaling,muennighoff2025s1,geiping2025scaling,zhou2025thinking}.

We build on the growing consensus that natural language is a particularly powerful medium for inference-time improvement. Natural language traces enable models to reason effectively in complex environments~\citep{lampinen2022tell,wei2022chain}, and language models can reliably map textual instructions to concrete modifications~\citep{chen2021evaluating,austin2021program,saunders2022self,scheurer2023training,feng2024natural}. However, existing methods often rely on random sampling of self-generated critiques, which may be noisy or fail to capture external preferences. In contrast, we leverage external rationales as directional information, enabling guided search in the semantic space.

\section{Experiments}
\label{sec:experiments}

We evaluate Feedback Descent across three diverse domains---visual design, prompt optimization, and molecule discovery---to demonstrate its generality and effectiveness.
Our evaluation spans diverse representations and evaluation modalities:
\textbf{(1) Representation diversity:} SVG (spatial/geometric), prompts (instructional text), molecules (chemical structures).
\textbf{(2) Evaluation modality:} SVG uses vision-language model comparison, prompt optimization uses dataset-specific accuracy metrics, molecules use computational docking scores.
Together, our experiments answer the question: \emph{Can a single optimization framework, with no domain-specific engineering, match or exceed specialized methods purely through structured feedback?}

\subsection{Experimental Domains}
We describe each evaluation domain and how we obtain pairwise comparisons augmented with textual rationales.

\begin{figure}[t]
\vspace{-3mm}
\centering
\includegraphics[width=0.82\linewidth]{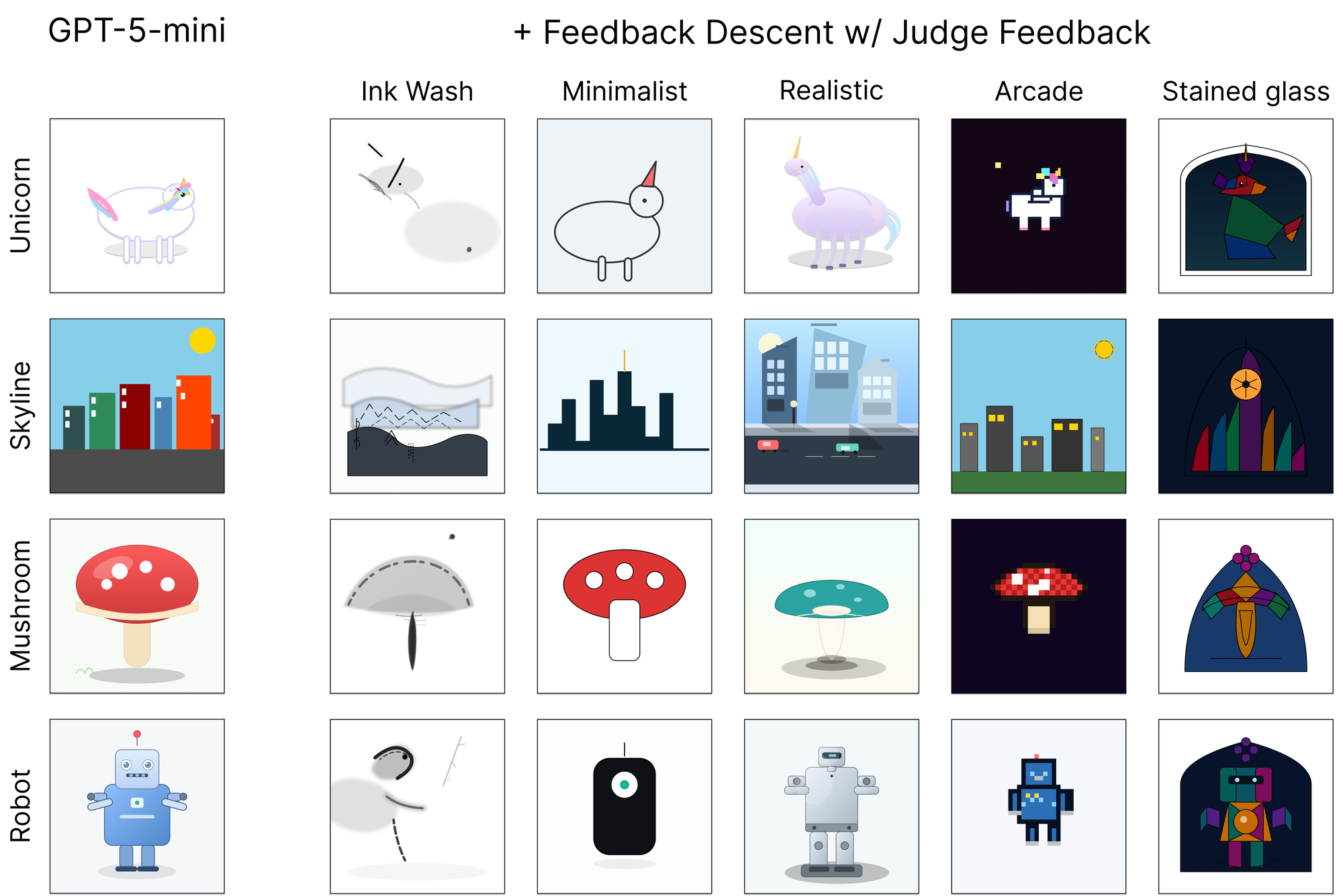}
\vspace{-2mm}
\caption{Example images generated by Feedback Descent under six different judge criteria. \textbf{Feedback Descent yields visually distinct objects aligned with the aesthetic criteria preferred by each judge.}}
\label{fig:visual_design_examples}
\end{figure}

\begin{table}[t]
\centering
\begin{tabular}{ll| *{5}{S[table-format=3.1, table-space-text-post=\%]}}
\toprule
\textbf{Setup} & \textbf{Subject} & {\textbf{Ink Wash}} & {\textbf{Minimalist}} & {\textbf{Realistic}} & {\textbf{Arcade}} & {\textbf{Stain Glass}} \\
\midrule
\multirow{4}{*}{\textbf{From Scratch}}
& City Skyline & 100.0 \% & 100.0 \% & 100.0 \% & 100.0 \% & 100.0 \% \\
& Mushroom & 100.0 \% & 100.0 \% & 83.3 \% & 100.0 \% & 100.0 \% \\
& Robot & 100.0 \% & 100.0 \% & 100.0 \% & 100.0 \% & 100.0 \% \\
& Unicorn & 100.0 \% & 100.0 \% & 100.0 \% & 100.0 \% & 100.0 \% \\
\midrule
\multirow{4}{*}{\textbf{Judge-Aware}}
& City Skyline & 50.0 \% & 100.0 \% & 100.0 \% & 100.0 \% & 88.9 \% \\
& Mushroom & 100.0 \% & 100.0 \% & 100.0 \% & 100.0 \% & 94.1 \% \\
& Robot & 100.0 \% & 100.0 \% & 100.0 \% & 93.8 \% & 100.0 \% \\
& Unicorn & 89.5 \% & 100.0 \% & 100.0 \% & 93.8 \% & 80.0 \% \\
\bottomrule
\end{tabular}
\caption{
\label{tab:visual_design}
Win rates after five iterations comparing Feedback Descent against direct prompting under two conditions: \textit{From Scratch} and \textit{Aware} of the judge rubric. Feedback Descent matches or outperforms the baseline on all combinations tested (i.e. $\geq 50\%$). \textbf{Iterative feedback consistently improves SVG designs over direct prompting.}
}
\end{table}

\textbf{SVG optimization.}
Taking inspiration from~\citet{bubeck2023sparks}, we evaluate the ability of models to output SVG code for illustrations of unicorns.
We use a set of five diverse judge prompts, each preferring a different aesthetic: \textit{ink wash} painting style, \textit{minimalist} design, \textit{realism}, \textit{retro arcade} pixel-art motifs, and \textit{stained glass} artwork.
We compare rendered SVGs using \texttt{GPT-5-mini}, which outputs both a binary preference and short textual feedback.
To mitigate order bias, we perform two judgments with swapped image orders (A-B and B-A) and declare a winner only if both judgments are consistent.
Otherwise, we try again, up to three times, and discard if no consistent winner emerges.

\textbf{Prompt optimization.}
We follow the setup of GEPA~\citep{agrawal2025gepa} across four diverse tasks: multi-hop reasoning (HotpotQA;~\citet{yang2018hotpotqa}), instruction following (IFBench;~\citet{pyatkin2025generalizing}), privacy-aware delegation (PUPA;~\citet{li2025papillon}), and retrieval-augmented verification (HoVer;~\citet{jiang2020hover}).
We evaluate on both open-source (Qwen3-8B;~\citet{yang2025qwen3}) and proprietary (GPT-4.1 mini) models.
For each task, we use the same multi-stage programs from GEPA, where the number of stages differs across datasets, and we jointly optimize the prompts for all stages using Feedback Descent.
Optimization is driven by training examples: candidate prompts are updated based on performance on the training set and textual feedback describing which constraints were satisfied or violated.
All candidate prompts are scored on validation examples, and the prompt with the highest validation accuracy rate is selected.
We report performance on held-out test examples.

\textbf{Molecule discovery.}
We evaluate on molecular docking tasks using DOCKSTRING~\citep{garciaortegon2021dockstring} docking scores and drug-likeness (QED).
DOCKSTRING provides a realistic drug discovery setting where molecules are evaluated based on their predicted binding affinity to medically relevant targets, rather than relying solely on simple physicochemical properties.
We focus on challenging optimization tasks across six protein targets: ADRB1, PGR, PPARA, PPARG, CDK2, and F2.
Following DOCKSTRING, we compute the combined score $s = -\textrm{Vina} - 10 \times (1 - \textrm{QED})$.
We represent molecules as SMILES strings~\citep{weininger1988smiles} and evaluate using DOCKSTRING's molecular docking pipeline to compute Vina scores (binding affinity).
The feedback system provides rich structured information, including RDKit molecular descriptors~\citep{landrum2006rdkit}, similarity searches against known compounds from molecular databases~\citep{liu2007bindingdb,gilson2016bindingdb,gaulton2012chembl,mendez2019chembl}, and detailed docking results.
In the system prompt, we also provide the LLM information about the protein target obtained from the UniProt database~\citep{uniprot2023}.
Together, this provides the LLM with detailed feedback on molecular properties that affect binding affinity, drug-likeness violations, and comparisons to known active compounds.

\subsection{SVG Optimization}

We evaluate iterative feedback against direct prompting across two generators, \texttt{GPT-4o-mini} and \texttt{GPT-5-mini}. The direct prompting baseline receives the full evaluation rubric and is tasked with producing a single best design. Feedback Descent instead begins with an initial set of candidates, and through 5 rounds of structured feedback and improvement, refines designs using judge comparisons that reflect aesthetic criteria. We test two initialization regimes: \textbf{Scratch}, which starts from images simply instructed to generate images of unicorns, and \textbf{Informed}, which starts from the strongest direct generations conditioned on the rubric, determined by the LLM judge.

\begin{table}[t]
\centering
\begin{tabular}{lcccccccc}
\toprule
\textbf{Method} & \multicolumn{4}{c}{\textbf{Qwen3-8B}} & \multicolumn{4}{c}{\textbf{GPT-4.1 Mini}} \\
\cmidrule(lr){2-5} \cmidrule(lr){6-9}
& \textbf{HpQA} & \textbf{IFBench} & \textbf{Hover} & \textbf{PUPA} & \textbf{HpQA} & \textbf{IFBench} & \textbf{Hover} & \textbf{PUPA} \\
\midrule
DSPy Default~\citep{khattab2024dspy} & 42.33 & 36.90 & 35.33 & 80.82 & 38.00 & 47.79 & 46.33 & 78.57 \\
MIPROv2~\citep{opsahl2024optimizing} & 55.33 & 36.22 & 47.33 & 81.55 & 58.00 & 49.15 & 48.33 & 83.37 \\
GRPO~\citep{shao2024deepseekmath} & 43.33 & 35.88 & 38.67 & 86.66 & --- & --- & --- & --- \\
GEPA~\citep{agrawal2025gepa} & \bf 62.33 & 38.61 & 52.33 & \bf 91.85 & \bf 69.00 & 52.72 & 51.67 & \bf 94.47 \\
\midrule
Feedback Descent (ours) & 60.00 & \bf 38.78 & \bf 60.00 & 90.90 & 68.33 & \bf 54.59 & \bf 57.67 & 85.66 \\
\bottomrule
\end{tabular}
\caption{
\label{tab:prompt_optimization}
Prompt optimization results across multiple benchmarks. Feedback Descent consistently outperforms or is competitive with state-of-the-art methods.}
\end{table}

\textbf{Results.} \cref{tab:visual_design} shows the win rates after 5 iterations. For both \texttt{GPT-4o-mini} and \texttt{GPT-5-mini}, Feedback Descent reliably improves outputs over the initial population.
Furthermore, qualitative examples of an optimization trajectory in~\cref{fig:visual_progression} and each judge-object pair in~\cref{fig:visual_design_examples} demonstrate that the procedure consistently produces unicorns whose visual style diverges across judges, aligning with aesthetic criteria such as geometry, minimalism, or retro arcade motifs.

\begin{takeawaybox}{Iterative feedback can elicit better outputs from the same model}
Because of a generator-verifier gap, even prompting with the exact judge rubric is suboptimal for SVG generation.
Feedback Descent elicits better images from the same generator by iteratively proposing improvements guided by feedback.
\end{takeawaybox}

\subsection{Prompt Optimization}

We compare Feedback Descent against five baselines: the default prompt implemented in the DSPy program~\citep[Default]{khattab2024dspy}, a Bayesian optimization approach for selecting instructions and demonstrations~\citep[MIPROv2]{opsahl2024optimizing}, online reinforcement learning~\citep[GRPO]{shao2024deepseekmath}, and a reflective prompt evolution method~\citep[GEPA]{agrawal2025gepa}.
All baselines are run under matched rollout budgets for fair comparison, and the reported baseline results are from~\citet{agrawal2025gepa}.

Each example produces pointwise feedback about which constraints were satisfied or violated.
To construct the pairwise feedback for Feedback Descent, we stratify the examples into quadrants based on whether each prompt resulted in a correct response.
We then ask the model to propose textual descriptions of inputs where these discrepancies arise.
We then statistically validate each hypothesis, filtering for ones that correspond to consistent differences in performance between the prompts.
This process distills the true global differences between the two prompts.

\Cref{tab:prompt_optimization} shows that Feedback Descent is competitive with GEPA across both models, achieving the best performance on IFBench and Hover, while GEPA leads on HotpotQA and PUPA.
Despite GEPA being specifically engineered for prompt optimization with coordinate descent and Pareto frontier maintenance, Feedback Descent achieves competitive performance with a simpler approach: jointly optimizing all prompts at once via automated textual summaries of pairwise performance differences.
We do not claim to present a state-of-the-art prompt optimizer; rather, these results demonstrate that our general-purpose framework remains competitive with specialized methods while requiring minimal domain-specific engineering.

\begin{takeawaybox}{Pairwise Comparisons Suffice for Competitive Prompt Optimization}
Feedback Descent is competitive with GEPA, a state-of-the-art prompt optimization method, despite employing a simpler, domain-agnostic approach.
\end{takeawaybox}

\begin{table}[t]
\centering
\begin{tabular}{ll S[table-format=2.3] S[table-format=1.3] S[table-format=1.3] S[table-format=2.3] S[table-format=1.3] S[table-format=1.3] | S[table-format=1.3]}
\toprule
& \textbf{Method} & {ADRB1} & {PGR} & {PPARA} & {PPARG} & {CDK2} & {F2} & {Avg} \\
\midrule
\multirow{6}{*}{\rotatebox[origin=c]{90}{\shortstack{DOCKSTRING \\ (N=260155)}}}
& Top $50\%$ & 5.305 & 3.478 & 4.549 & 4.210 & 4.385 & 4.168 & 4.349 \\
& Top $90\%$ & 8.785 & 7.878 & 7.987 & 7.658 & 7.733 & 7.477 & 7.920 \\
& Top $99\%$ & 9.620 & 8.703 & 8.718 & 8.449 & 8.453 & 8.139 & 8.680 \\
& Top $99.9\%$ & 10.209 & 9.260 & 9.230 & 9.012 & 8.979 & 8.722 & 9.235 \\
& Top $99.99\%$ & {\blueul{10.742}} & {\blueul{9.723}} & 9.821 & 9.518 & 9.509 & 9.252 & 9.761 \\
& Best Molecule & {\blueul{11.330}} & {\blueul{9.742}} & 9.907 & 9.529 & 9.534 & {\blueul{9.311}} & 9.892 \\
\midrule
& Graph MCTS$^\dagger$~\citep{jensen2019graph} & 8.883 & 7.819 & 7.363 & 7.134 & 7.777 & 6.310 & 7.548 \\
& SMILES GA~\citep{brown2019guacamol} & 9.334 & 8.335 & 9.052 & 8.560 & 8.268 & 7.984 & 8.589 \\
& REINVENT~\citep{olivecrona2017molecular} & 9.867 & 8.604 & 8.735 & 9.054 & 8.695 & 8.441 & 8.899 \\
& Graph GA$^\dagger$~\citep{jensen2019graph} & 10.249 & 8.793 & 9.211 & 8.769 & 8.652 & 8.900 & 9.096 \\
& GP-BO$^\dagger$~\citep{tripp2021a} & 10.552 & 9.307 & 9.680 & 9.485 & 9.067 & 8.686 & 9.463 \\
\midrule
& TextGrad~\citep{yuksekgonul2024textgrad} & 8.531 & 8.057 & 7.953 & 7.256 & 8.174 & 7.357 & 7.888 \\
& Feedback Descent (ours) & \bfseries 10.623 & \bfseries 9.615 & \bfseries 9.919 & \bfseries 10.187 & \bfseries 9.803 & \bfseries 9.300 & \bfseries 9.908 \\
& \hspace{2mm} w/ No Feedback & 6.190 & 8.619 & 8.230 & 8.633 & 8.300 & 8.793 & 8.127 \\
& \hspace{2mm} w/ Random Feedback & 6.604 & 8.385 & 8.276 & 6.780 & 8.793 & 7.993 & 7.805 \\
& \hspace{2mm} w/ Binary Only & 5.863 & 8.779 & 8.507 & 7.998 & 9.439 & 8.420 & 8.168 \\
\bottomrule
\end{tabular}
\vspace{-2mm}
\captionof{table}{
\label{tab:dockstring}
Results for molecule optimization on six protein targets. Full results with standard deviations are in~\cref{tab:dockstring_full}.
For each target, the top generative result is in \textbf{bold}, and any population in the DOCKSTRING database that exceeds the best generative result is {\blueul{underlined}}.
\textbf{Feedback Descent rivals or surpasses specialized molecular optimizers across all six targets.}
}
\vspace{-2mm}
\end{table}

\subsection{Molecule Optimization (DOCKSTRING)}

\begin{wrapfigure}{R}{0.4\linewidth}
\vspace{-3mm}
\centering
\begin{tabular}{r S[table-format=2.2] S[table-format=1.2] S[table-format=2.2]}
\toprule
\textbf{Noise} & {\textbf{ADRB1}} & {\textbf{PGR}} & {\textbf{PPARG}} \\
\midrule
None & \bfseries 10.62 & \bfseries 9.62 & \bfseries 10.19 \\
25\% & 9.28 & 9.14 & 8.16 \\
50\% & 10.21 & 8.92 & 8.75 \\
100\% & 6.60 & 8.39 & 6.78 \\
\bottomrule
\end{tabular}
\vspace{-1mm}
\captionof{table}{
\label{tab:ablation_noisy_feedback}
DOCKSTRING scores for ADRB1, PGR, and PPARG with Feedback Descent at varying feedback noise levels.
\textbf{Performance degrades gracefully with increasing noise.}
}

\vspace{-3mm}
\end{wrapfigure}
We compare against baselines implemented in the \texttt{mol\_opt} repository~\citep{gao2022molopt},
Our comparisons include a genetic algorithm~\citep[SMILES GA]{brown2019guacamol}, reinforcement learning~\citep[REINVENT]{olivecrona2017molecular}, fragment-based algorithms~\citep[Graph MCTS/GA]{jensen2019graph}, and Bayesian optimization on molecular graphs~\citep[GP-BO]{tripp2021a}.
Because fragment-based methods exploit graph-level structural priors, the most direct comparison is to the text-only baselines: SMILES-GA and REINVENT.
We also compare our approach with TextGrad~\citep{yuksekgonul2024textgrad}, a recent work that similarly utilizes an LLM to make textual updates to SMILES strings.
The key difference is that TextGrad's improvement proposals are pointwise, conditioning only on the latest molecule, whereas Feedback Descent conditions on accumulated feedback history, enabling more effective continual improvement at high iteration budgets.

\textbf{Main Results.}
\begin{wrapfigure}{R}{0.27\linewidth}
\vspace{-3mm}
\centering
\begin{tabular}{lc}
\toprule
\textbf{Target} & \textbf{Win Rate} \\
\midrule
ADRB1 & 76\% \\
CDK2 & 86\% \\
F2 & 79\% \\
PPARG & 83\% \\
\midrule
Avg & \textbf{81\%} \\
\bottomrule
\end{tabular}
\captionof{table}{
\label{tab:feedback_alignment}
True feedback vs.\ scrambled feedback win rate across 400 comparisons.
}

\vspace{-3mm}
\end{wrapfigure}
\cref{tab:dockstring} summarizes optimization outcomes across six protein targets.
Feedback Descent outperforms all baselines and achieves the strongest scores on all targets.
On multiple proteins, it matches or exceeds the 99.9th and even 99.99th percentiles of the DOCKSTRING database, including surpassing the best molecule present in the database itself ($N=260155$).
Notably, TextGrad consistently underperforms Feedback Descent across all targets; while the TextGrad paper evaluated on similar DOCKSTRING tasks for only $\sim$10 iterations, we run both methods for 1000 steps and find that TextGrad's pointwise conditioning does not scale well to high iteration budgets, confirming that accumulated feedback history is essential for sustained improvement.
These findings show that Feedback Descent, a purely text-based method, can rival or outperform specialized graph-based algorithms, despite lacking handcrafted structural priors.

\cref{fig:molecular_optimization} shows optimization trajectories for PPARG.
Feedback Descent achieves competitive trajectories relative to specialized methods, reaching high-scoring regions of chemical space with comparable or fewer oracle calls.
This pattern holds across targets, suggesting that the method generalizes rather than relying on idiosyncrasies of a single protein system.

\textbf{Analysis of discovered molecules.}
\cref{fig:pareto_frontier} illustrates the Pareto frontier between docking affinity (Vina score) and drug-likeness (QED) for PPARG. Feedback Descent recovers molecules that sit on or above the DOCKSTRING frontier, indicating that improvements in affinity are achieved without compromising drug-likeness. See \cref{fig:pareto_full} in the appendix for the full set of Pareto frontiers across all targets. These results show that feedback-guided search yields candidates that are not only potent but also balanced along multiple drug-relevant dimensions.

We also examine novelty by plotting Tversky similarity (CFP4 fingerprints) to approved DrugBank molecules against docking scores in~\cref{fig:novelty}. Across all targets, the correlations are weak or negative (Spearman $\rho$ between $-0.39$ and $0.40$), indicating that high-scoring candidates identified by Feedback Descent do not simply recycle functional groups from existing drugs, but instead explore novel regions of chemical space. For CDK2, no comparison is shown, as the target lacks any fully approved drugs in DrugBank with orthosteric binding as part of their mechanism of action, and thus does not satisfy our inclusion criteria.

\begin{figure}[t]
\vspace{-3mm}
\centering
\begin{minipage}{0.6\linewidth}
\vspace{-22mm}
\centering \includegraphics[width=\linewidth]{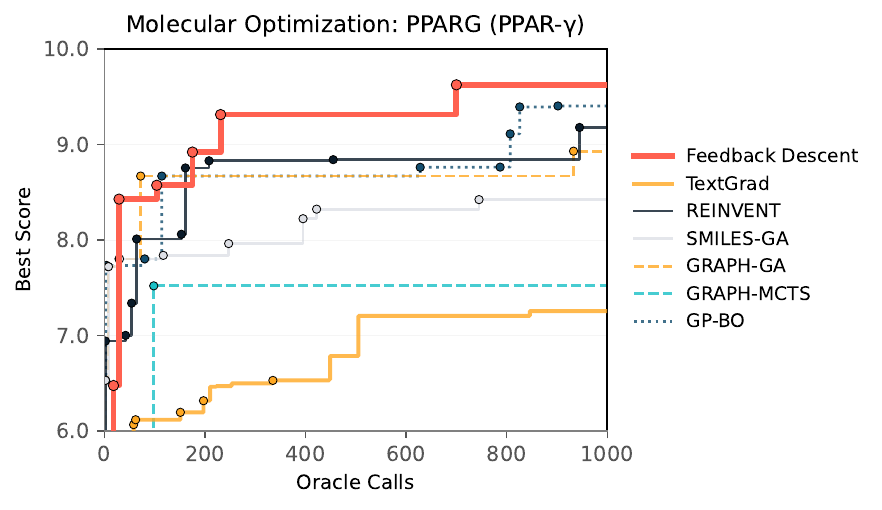}
\vspace{-6mm}
\captionof{figure}{
\label{fig:molecular_optimization}
Optimization trajectories on PPARG showing docking scores over oracle calls for Feedback Descent and specialized baselines.
\textbf{Feedback Descent quickly improves molecular docking scores within the first few hundred oracle calls.}
}
\end{minipage}
\hfill
\begin{minipage}[t]{0.37\linewidth}
\centering \includegraphics[width=\linewidth]{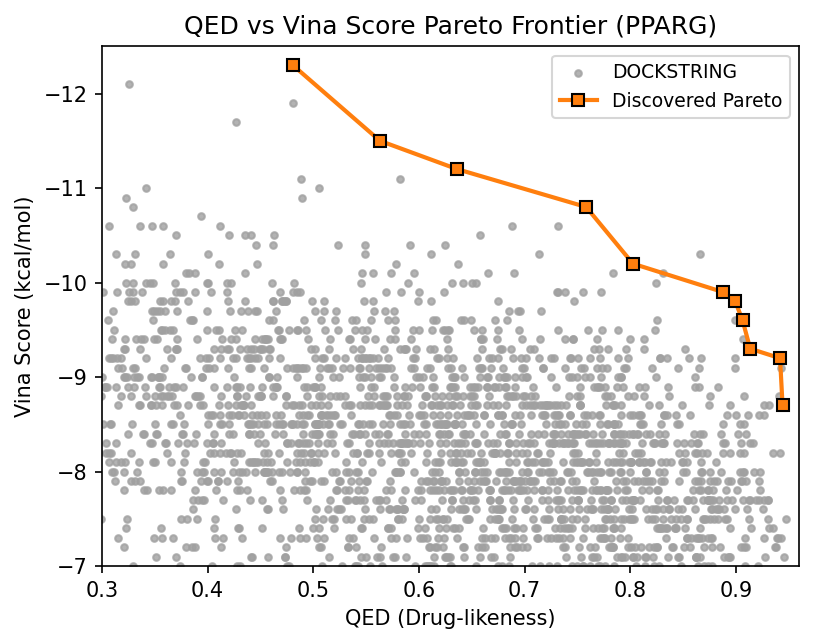}
\vspace{-4mm}
\captionof{figure}{
\label{fig:pareto_frontier}
Pareto frontier of docking affinity vs. drug-likeness, comparing Feedback Descent molecules (blue) to the DOCKSTRING database (gray).
\textbf{Feedback Descent finds novel molecules that meet or surpass known ones.}
}
\end{minipage}
\vspace{-3mm}
\end{figure}

\begin{figure}[t]
\vspace{-2mm}
\centering
\includegraphics[width=.50\linewidth]{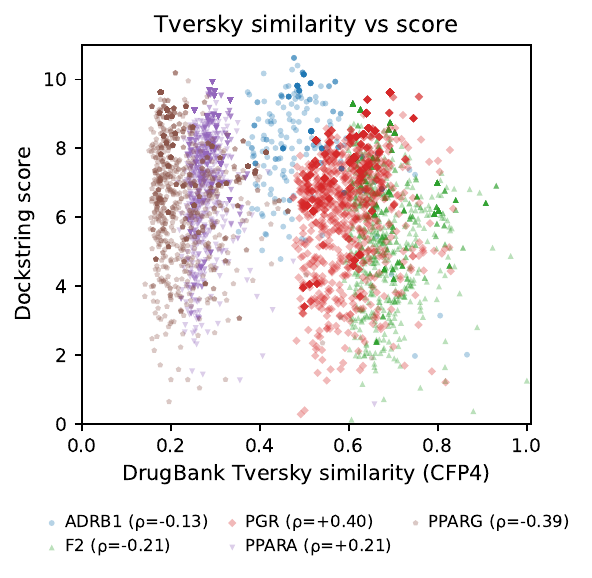}
\vspace{-2mm}
\caption{
\label{fig:novelty}
Scatter plots of Tversky similarity to approved drugs against docking scores, showing weak or negative correlations across targets. \textbf{High-scoring molecules discovered by Feedback Descent are far from any known drugs.}
}
\end{figure}

\textbf{Feedback quality and type.}
In addition to the main points of comparison, we show three ablations on the quality of the feedback in~\cref{tab:dockstring}.
We evaluate (1) \textbf{No Feedback}, i.e., parallel best-of-N sampling with no feedback; (2) \textbf{Random Feedback}, which uses the same iterative algorithm but shuffles rationales between molecule pairs; and (3) \textbf{Binary Only}, which provides only the binary signal of whether the newest molecule is better than the previous candidate.
Random feedback underperforms even best-of-N, confirming that the method relies on feedback content rather than just the iterative structure.
The gap between Binary Only and Feedback Descent (1.74 docking score units) demonstrates that textual rationales provide substantial value beyond the binary preference signal.
\Cref{tab:ablation_noisy_feedback} further tests robustness to corrupted feedback by randomly shuffling rationales at varying noise levels.
The method degrades gracefully: 50\% noise still achieves scores above the 99.9th percentile of DOCKSTRING for ADRB1 and PGR.

\textbf{Feedback alignment.}
Finally, we test whether generated molecules actually reflect the feedback shown during optimization in~\cref{tab:feedback_alignment}.
We present an LLM judge with a newly generated molecule and ask it to determine whether the molecule more closely follows the true feedback or a scrambled control.
We randomly flip the order shown to the judge to account for order bias.
Across 400 comparisons, true feedback wins 81\% of head-to-head comparisons ($p < 10^{-10}$, binomial test), demonstrating that the generator reliably reads and incorporates structured property feedback.

\begin{takeawaybox}{Feedback Descent Can Discover Novel Targeted Molecules}
Feedback Descent, operating in a purely textual form, consistently identifies novel molecules that surpass high-percentile baselines in DOCKSTRING. This demonstrates that iterative, feedback-guided optimization can enable models to genuinely explore unknown design spaces beyond their training distribution.
\end{takeawaybox}

\section{Discussion}

This paper presents Feedback Descent, an inference-time framework that improves text artifacts through structured pairwise feedback. We validate it on visual design, prompt optimization, and molecule discovery, showing that text can serve as an optimizable medium, not just static data. Unlike parameter tuning, this approach can leverage richer textual signals, allowing for continual improvement without requiring retraining.

\textbf{Limitations}. The method relies on strong evaluators, which may be scarce in some domains. Training models to produce reliable feedback remains a prerequisite for harder tasks. For creative domains, strictly “following the gradient” may be limiting; balancing refinement with exploration is an important next step.

\newpage
\subsubsection*{Acknowledgments}
We would like to thank Wanqiao Xu, Allen Nie, Henrik Marklund, Mert Yuksekgonul, Anikait Singh, and Jubayer Ibn Hamid for helpful discussions and feedback.
This work is supported by the OpenAI Superalignment Fellowship, the KFAS Fellowship, and the Schmidt Sciences AI2050 program.

\bibliography{citations}

\clearpage
\appendix

\section{Formal Statements and Proofs}
\label{app:proof}

\setcounter{proposition}{0}
\begin{proposition}[Linear convergence under PL with rationale-guided directions]
\label{prop:gradient-linear-full}
Let $r: Z \to \mathbb{R}$ be $L$-smooth and satisfy the $\mu$-PL condition (for maximization)
\[
\frac{1}{2}\|\nabla r(z)\|_2^2 \;\ge\; \mu\big(r(z^\star)-r(z)\big)\qquad \forall z\in Z.
\]
At iteration $t$, suppose a direction $v_t$ satisfies
\[
\mathbb{E}[\,v_t \mid z_t\,] \;=\; \alpha\,\nabla r(z_t),\qquad
\mathbb{E}\!\left[\,\|v_t-\mathbb{E}[v_t\mid z_t]\|_2^2 \,\big|\, z_t\right] \;\le\; \sigma^2 \|\nabla r(z_t)\|_2^2,
\]
with constants $\alpha>0$ and $\sigma \ge 0$, and define $\kappa_1 \triangleq \alpha^2+\sigma^2$.
Consider the update $z_{t+1}=z_t+\eta v_t$. If a constraint set $Z$ is present, assume $z_t+\eta v_t\in Z$ (i.e., the projection is inactive).
With stepsize $\eta=\alpha/(L\kappa_1)$,
\[
\mathbb{E}\!\left[r(z^\star)-r(z_{t+1}) \,\big|\, z_t\right]
\;\le\;
\Big(1-\tfrac{\mu\,\alpha^2}{L\,\kappa_1}\Big)\,\big[r(z^\star)-r(z_t)\big].
\]
Unrolling yields
\[
\mathbb{E}\!\left[r(z^\star)-r(z_T)\right]\;\le\;
\Big(1-\tfrac{\mu\,\alpha^2}{L\,\kappa_1}\Big)^T \big[r(z^\star)-r(z_0)\big],
\]
so $\epsilon$-accuracy is achieved in
\[
T \;=\; O\!\left(\frac{L(\alpha^2+\sigma^2)}{\mu\,\alpha^2}\,\log\frac{1}{\epsilon}\right)
\]
iterations.
\end{proposition}
\begin{proof}
$L$-smoothness gives the two-sided bound
\[
r(z_t+\eta v_t) \;\ge\; r(z_t) + \eta \langle \nabla r(z_t), v_t\rangle - \tfrac{L}{2}\eta^2 \|v_t\|_2^2.
\]
Taking conditional expectation and using $\mathbb{E}[v_t|z_t]=\alpha \nabla r(z_t)$ and
$\mathbb{E}\!\left[\|v_t\|_2^2 \,\big|\, z_t\right] \le (\alpha^2+\sigma^2)\,\|\nabla r(z_t)\|_2^2=\kappa_1\|\nabla r(z_t)\|_2^2$,
\[
\mathbb{E}\!\left[r(z_{t+1}) \mid z_t\right]
\;\ge\;
r(z_t) + \Big(\eta\alpha - \tfrac{L}{2}\eta^2\kappa_1\Big)\,\|\nabla r(z_t)\|_2^2.
\]
By the PL inequality, $\|\nabla r(z_t)\|_2^2 \ge 2\mu\,[r(z^\star)-r(z_t)]$, so
\[
\mathbb{E}\!\left[r(z^\star)-r(z_{t+1}) \mid z_t\right]
\;\le\;
\Big(1 - 2\mu\eta\alpha + \mu L \eta^2 \kappa_1\Big)\,[r(z^\star)-r(z_t)].
\]
Choosing $\eta=\alpha/(L\kappa_1)$ makes the bracket equal to $1 - \mu\alpha^2/(L\kappa_1)$, yielding the claim.
\end{proof}

\subsection{Query Complexity and Dimension Dependence}

\textbf{Dimension-Free Case.} When rationales provide full gradient information ($v_t \in \mathbb{R}^d$) at unit cost, the query complexity equals $T$ and is dimension-independent:
\begin{align}
\text{Queries} = O\left(\frac{L(\alpha^2 + \sigma^2)}{\alpha^2 \mu} \log \frac{1}{\epsilon}\right)
\end{align}

\textbf{Coordinate-Sparse Case.} Suppose each query reveals one coordinate of $\nabla r(z_t)$ chosen uniformly at random. Using the unbiased estimator $v_t = d\,(\partial_{i} r(z_t))\,e_i$ with $i\sim\mathrm{Unif}([d])$ gives $\alpha=1$, $\sigma^2=d-1$, and hence $\kappa_1=d$ and stepsize $\eta=1/(Ld)$. We have
\[
T = O\!\Big(\frac{L d}{\mu}\log\frac{1}{\epsilon}\Big),\qquad
\text{Queries} = O\!\Big(\frac{L d}{\mu}\log\frac{1}{\epsilon}\Big).
\]
Equivalently, averaging $m$ independent coordinate queries per iteration yields $\sigma^2=(d-1)/m$; taking $m=d$ recovers $T=O((L/\mu)\log(1/\epsilon))$ with $d$ queries per iteration, so total queries remain $\Theta\!\big(\frac{L d}{\mu}\log\frac{1}{\epsilon}\big)$.

This clarifies when and why dimension appears in the complexity.

\section{Lower Bounds for Exhaustive/Random Zeroth-Order Search}
\label{sec:exhaustive-lb}

We formalize the intrinsic slowness of exhaustive (grid) search and best-of-$N$ random sampling when only function values (or preferences) are used without directional information. The hard instance is the strongly concave quadratic
\[
r(z) \;=\; r(z^\star) - \tfrac{\mu}{2}\,\|z-z^\star\|_2^2,
\qquad z\in B_R(z^\star)\subset\mathbb{R}^d,
\]
whose $\epsilon$-optimal set is the ball $B_{\rho_\epsilon}(z^\star)$ with radius $\rho_\epsilon=\sqrt{2\epsilon/\mu}$.

\begin{proposition}[Grid-search lower bound]
\label{prop:grid}
Let $B_R(z^\star)\subset\mathbb{R}^d$ and a hypercubic grid of spacing $h$. Its covering radius is $\rho=\frac{\sqrt{d}\,h}{2}$. To guarantee that \emph{for all} placements of $z^\star$ there exists a grid point in the $\epsilon$-optimal ball $B_{\rho_\epsilon}(z^\star)$ with $\rho_\epsilon=\sqrt{2\epsilon/\mu}$, it suffices that $\rho\le\rho_\epsilon$ (i.e., $h\le 2\rho_\epsilon/\sqrt d$). Furthermore, any such grid restricted to $B_R(z^\star)$ must contain at least
\[
N \;\ge\; \left(\frac{R}{\rho}\right)^{\!d}
\;=\; \Big(\frac{R\sqrt d}{2\rho_\epsilon}\Big)^{\!d}
\;=\; \Big(\frac{\mu R^2 d}{8\,\epsilon}\Big)^{\!d/2}
\]
points. Hence exhaustive grid search is exponential in $d$ and polynomial in $1/\epsilon$ with exponent $d/2$ on this family.
\end{proposition}

\begin{proof}
Coverage of $B_R(z^\star)$ by $N$ balls of radius $\rho$ centered at grid points implies $N V_d \rho^d \ge V_d R^d$, hence $N\ge (R/\rho)^d$. With $\rho=\sqrt d\,h/2$ and $h\le 2\rho_\epsilon/\sqrt d$, we obtain $N\ge (R\sqrt d/(2\rho_\epsilon))^d$. Substitute $\rho_\epsilon=\sqrt{2\epsilon/\mu}$ to conclude.
\end{proof}

\begin{proposition}[Best-of-$N$ random sampling lower bound]
\label{prop:random-fixed}
Draw $X_1,\dots,X_N \stackrel{\text{i.i.d.}}{\sim}\mathrm{Unif}(B_R(z^\star))$ and let $\hat z=\arg\max_i r(X_i)$ for
$r(z)=r(z^\star)-\tfrac{\mu}{2}\|z-z^\star\|_2^2$. Then with $a\triangleq 2/d$,
\[
\mathbb{E}\!\left[r(z^\star)-r(\hat z)\right]
=\frac{\mu R^2}{2}\,N\,\mathrm{B}(1{+}a,N)
=\frac{\mu R^2}{2}\,\Gamma(1{+}a)\,\frac{\Gamma(N+1)}{\Gamma(N+1+a)}.
\]
Moreover, for all $d\ge 1$ (so $a\in(0,2]$),
\[
\frac{\Gamma(N+1)}{\Gamma(N+1+a)} \ \ge\ (N+2)^{-a},
\]
and thus
\[
\mathbb{E}\!\left[r(z^\star)-r(\hat z)\right]
\ \ge\ \frac{\mu R^2}{2}\,\Gamma\!\Big(1+\frac{2}{d}\Big)\,(N+2)^{-\frac{2}{d}}
\ =\ \Omega\!\big(N^{-\frac{2}{d}}\big).
\]
\end{proposition}

\begin{proof}
Let $R_i=\|X_i-z^\star\|_2$ and $R_{\min}=\min_i R_i$. The CDF of $R_{\min}$ is
$F(r)=1-(1-(r/R)^d)^N$ for $r\in[0,R]$. Differentiating,
$f(r)=N d r^{d-1} R^{-d}(1-(r/R)^d)^{N-1}$. Then
\[
\mathbb{E}[R_{\min}^2]
=\int_0^R r^2 f(r)\,dr
= N R^2 \int_0^1 t^{\frac{2}{d}}(1-t)^{N-1} dt
= N R^2\,\mathrm{B}\!\big(1+\tfrac{2}{d},\,N\big),
\]
where $t=(r/R)^d$ and $\mathrm{B}$ is the Beta function. Using
$\mathrm{B}(a,b)=\frac{\Gamma(a)\Gamma(b)}{\Gamma(a+b)}$ gives the exact expression. For the bound, we use the inequality $\Gamma(N+1)/\Gamma(N+1+a) \ge (N+2)^{-a}$ which holds for all $a\in(0,2]$ and $N\ge 1$.
\end{proof}

\section{Extended Experiment Section}

\subsection{Implementation Details}
\label{app:implementation}

\paragraph{SVG Code Optimization.} We employ a tournament-style approach where \texttt{gpt-5-mini} generates SVG/TikZ code that gets rendered to PNG images for pairwise aesthetic comparisons by a separate instance of the same model acting as judge. The system maintains a ``champion'' design that only updates when both A-vs-B and B-vs-A orderings consistently agree on a winner, accumulating winning rationales into the generation prompt to guide aesthetic improvements across iterations. The judge provides natural language rationales explaining aesthetic preferences that inform subsequent generations.

\paragraph{IFBench Prompt Optimization.}
We closely follow the setting of~\citet{agrawal2025gepa} for this experiment, including their choice of LLMs (Qwen3-8B and GPT-4.1-mini) and multi-stage DSPy programs.
For Qwen3-8B, we use 10 iterations with no early stopping (patience=0) with temperature 0.6, and evaluate 2 prompt candidates per round across 20 hard examples.
For GPT-4.1-mini, we increase to 15 iterations with early stopping (patience=5) with temperature 1.0 for exploration.
Both configurations use the same prompt improver template, described in~\cref{app:prompts}.

We programmatically generate a textual description of the difference between two prompts.
To compare two prompts, we first partition the training set into four quadrants based on outcomes: examples where prompt A succeeds and B fails (A\_wins), A fails and B succeeds (B\_wins), both fail (tie\_fail), or both succeed (tie\_success).
We then use the same LLM to propose hypotheses about input characteristics (based on the prompt text) and output patterns (based on the response and evaluation feedback), producing around 20 hypotheses for each category.
To evaluate whether each hypothesis applies to each example, we use the same LLM as the tagger that outputs binary labels (1 if the hypothesis matches, 0 otherwise) for all hypotheses in a single call per example, processing hundreds of examples in parallel.
We then compute lift metrics for each hypothesis-quadrant pair, where lift is the ratio of conditional probability to the base rate (i.e., how much more likely an outcome is given the hypothesis holds).
We validate the hypotheses statistically using Fisher's exact test, and filter for hypotheses that are statistically significant at $p < 0.1$ with minimum support of 3 examples and lift $\geq 2.0$ for A/B wins or $\geq 1.5$ for failures.
This analysis identifies which input patterns correlate with differential performance and which output characteristics appear when one prompt outperforms the other, providing actionable insights for prompt improvement.

\paragraph{Molecule Optimization.} We implement molecular optimization using the DOCKSTRING package~\citep{garciaortegon2021dockstring} for protein-ligand docking simulations across six therapeutic targets. The system begins with three simple seed molecules (acetamide, pentane, benzene) and progressively evolves SMILES strings through iterative feedback loops that incorporate RDKit molecular properties, protein binding site information, and similarity comparisons to approved drugs as metadata. We use the combined score function suggested by DOCKSTRING:
\begin{equation}
    s_{\text{overall}}(\text{molecule}, \text{protein}) = -\texttt{Vina}(\text{molecule}, \text{protein}) - 10 * (1 - \texttt{QED}(\text{molecule})),
\end{equation} where \texttt{Vina} provides the binding affinity prediction (kcal/mol, more negative is better) and the QED penalty term penalizes molecules with poor drug-likeness, with lower overall scores indicating better molecules that balance binding strength and drug-like properties. Note that QED scores range from $0$ to $1$ while Vina scores typically range from $-3.0$ to $-12.0$ kcal/mol. For Feedback Descent, we use a batch size of 8 and top-k selection of 10 examples.

\subsection{Additional Results}
\begin{figure}[t]
\centering \includegraphics[width=0.98\linewidth]{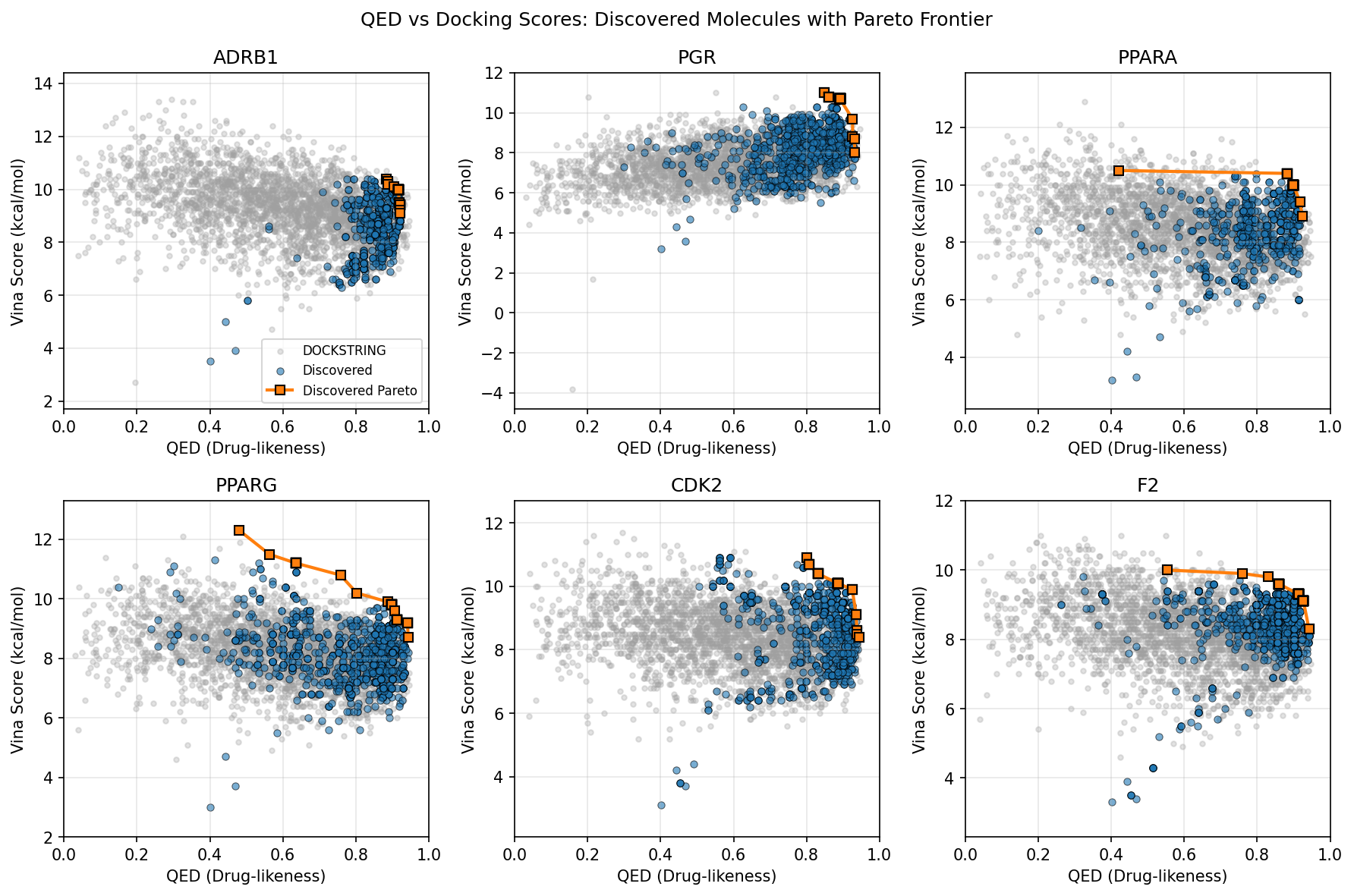}
\caption{
  \label{fig:pareto_full}
Pareto frontiers of discovered molecules (blue) compared against molecules in the DOCKSTRING dataset (gray) across six protein targets. The highlighted orange markers indicate molecules on the discovered Pareto frontier, achieving joint improvements in docking affinity (Vina score) and drug-likeness (QED).
  }
\end{figure}

\label{app:additional_results}

\cref{fig:pareto_full} shows that across all protein targets, the discovered molecules extend beyond the DOCKSTRING baseline along both axes. The resulting Pareto frontiers illustrate consistent improvements in the joint trade-off between docking affinity and drug-likeness, highlighting that feedback-guided search yields coordinated gains rather than isolated outliers.

\begin{table}[t]
\centering
\resizebox{\linewidth}{!}{%
\begin{tabular}{llccccccc}
\toprule
& \textbf{Method} & ADRB1 & PGR & PPARA & PPARG & CDK2 & F2\\
\midrule
\multirow{6}{*}{\rotatebox[origin=c]{90}{\shortstack{DOCKSTRING \\ (N=260155)}}}
& Top $50\%$ & 5.305 & 3.478 & 4.549 & 4.210 & 4.385 & 4.168 \\
& Top $90\%$ & 8.785 & 7.878 & 7.987 & 7.658 & 7.733 & 7.477 \\
& Top $99\%$ & 9.620 & 8.703 & 8.718 & 8.449 & 8.453 & 8.139 \\
& Top $99.9\%$ & 10.209 & 9.260 & 9.230 & 9.012 & 8.979 & 8.722 \\
& Top $99.99\%$ & \blueul{10.742} & \blueul{9.723} & 9.821 & 9.518 & 9.509 & 9.252 \\
& Best Molecule & \blueul{11.330} & \blueul{9.742} & 9.907 & 9.529 & 9.534 & \blueul{9.311} \\
\midrule
& GP-BO$^\dagger$~\citep{tripp2021a} & $10.552 \pm 0.140$ & $9.307 \pm 0.177$ & $9.680 \pm 0.337$ & $9.485 \pm 0.279$ & $9.067 \pm 0.289$ & $8.686 \pm 0.068$ \\
& Graph MCTS$^\dagger$~\citep{jensen2019graph} & $8.883 \pm 0.826$ & $7.819 \pm 0.319$ & $7.363 \pm 0.935$ & $7.134 \pm 0.855$ & $7.777 \pm 0.723$ & $6.310 \pm 0.704$ \\
& Graph GA$^\dagger$~\citep{jensen2019graph} & $10.249 \pm 1.002$ & $8.793 \pm 0.497$ & $9.211 \pm 0.343$ & $8.769 \pm 0.432$ & $8.652 \pm 0.449$ & $8.900 \pm 0.817$ \\
& SMILES GA~\citep{brown2019guacamol} & $9.334 \pm 0.237$ & $8.335 \pm 0.276$ & $9.052 \pm 0.484$ & $8.560 \pm 0.346$ & $8.268 \pm 0.170$ & $7.984 \pm 0.554$ \\
& REINVENT~\citep{olivecrona2017molecular} & $9.867 \pm 0.522$ & $8.604 \pm 0.483$ & $8.735 \pm 0.120$ & $9.054 \pm 0.153$ & $8.695 \pm 0.370$ & $8.441 \pm 0.535$ \\
\midrule
& No Feedback (Best-of-N) & $6.190 \pm 0.821$ & $8.619 \pm 0.562$ & $8.230 \pm 0.628$ & $8.633 \pm 0.549$ & $8.300 \pm 0.620$ & $8.793 \pm 0.921$ \\
& Random Feedback & $6.604 \pm 0.577$ & $8.385 \pm 0.258$ & $8.276 \pm 0.628$ & $6.780 \pm 0.523$ & $8.793 \pm 0.921$ & $7.993 \pm 0.663$ \\
& Minimal Feedback & $5.863 \pm 0.428$ & $8.779 \pm 0.633$ & $8.507 \pm 0.428$ & $7.998 \pm 0.571$ & $9.439 \pm 0.922$ & $8.420 \pm 0.315$ \\
& TextGrad~\citep{yuksekgonul2024textgrad} & $8.531 \pm 0.278$ & $8.057 \pm 0.383$ & $7.953 \pm 0.160$ & $7.256 \pm 0.886$ & $8.174 \pm 0.395$ & $7.357 \pm 0.821$ \\
& Feedback Descent & $\mathbf{10.623 \pm 0.112}$ & $\mathbf{9.615 \pm 0.158}$ & $\mathbf{9.919 \pm 0.305}$ & $\mathbf{10.187 \pm 0.253}$ & $\mathbf{9.803 \pm 0.267}$ & $\mathbf{9.300 \pm 0.062}$ \\
\bottomrule
\end{tabular}
}
\caption{
\label{tab:dockstring_full}
Full results for molecule optimization on six protein targets with standard deviations.
Fragment-based algorithms (denoted by $\dagger$) operate directly on molecular graphs, giving them structural priors unavailable to purely text-based methods.
For each target, the top generative result is in \textbf{bold}, and any population in the DOCKSTRING database that exceeds the best generative result is {\blueul{underlined}}.
}
\end{table}

\subsection{Prompt Templates}
\label{app:prompts}

We use the following prompt for the judge for the Anatomy SVG task. The rubrics for the other tasks are written in a similar style, translating a particular aesthetic into operational rules that minimize ambiguity.
\begin{promptbox}[Anatomy Judge Rubric]
RUBRIC NAME: Anatomical Realism
INTENT: Believable equine anatomy with a plausible horn; form, proportion, and structure matter most.

NON-NEGOTIABLES:
- Recognizable equine proportions; head, neck, torso, four legs, mane, tail, horn present.
- Limbs connect anatomically; joints and hooves indicated.

CRITICAL BENCHMARKS (must evaluate these first):
1. Head-Neck Proportion: Neck length should be ~1.5x head length; head meets neck high on shoulders
2. Body Square: Body length (shoulder to buttock) ~= height at withers; chest depth ~= elbow height
3. Leg Structure: Proper joint articulation with elbow under withers; fetlock/pastern angles 45-55 deg when standing; all four limbs distinct and correctly connected

WHAT TO REWARD:
- Correct limb count and articulation; mass distribution that could stand or move.
- Horn integrates naturally with the skull (frontal bone center, 2-3" above eye line).
- Subtle shading or line variation conveying volume.
- Ground contact or cast shadow for grounding.
- Visible muscle definition suggesting tension/relaxation appropriate to pose.
- Differentiated hair textures: short coat vs coarse mane/tail strands.
- Anatomical landmarks: withers prominence, gaskin curve.

WHAT TO PENALIZE:
- Missing or fused legs; impossible joints; balloon torsos.
- Flat cardboard profiles with no sense of volume.
- Decorative effects that obscure structure.
- Disney-fied proportions (oversized eyes, baby-like features).
- Horn placement anywhere except frontal bone center (2-3" above eye line).

TIEBREAKERS:
- Prefer the image with more accurate limb/neck/head proportions.
- If both are plausible, choose the one with better weight and grounding.
\end{promptbox}

We use the following prompt templates for candidate generation and rationale generation for prompt optimization.

\begin{promptbox}[System Prompt Template for Prompt Optimization]
Improve the assistant's prompt by extracting actionable insights from the data.

## Goal
Create prompts that generalize well beyond the training examples you see here. The patterns below come from a small sample; your output must work on thousands of unseen cases.

## Current Prompts
**Approach A (Baseline):**
```python
{prompt_a_dict}
```

**Approach B (Challenger):**
```python
{prompt_b_dict}
```

## Training Signals
{comparison}

## Prompt Improvement Strategy

**1. Extract Core Insights**
Identify patterns with strong evidence (low p-value, high lift, good support):
- What fundamental strategies distinguish success from failure?
- What misunderstandings or errors repeatedly occur?
- Are there essential facts or constraints the model needs to know?

**2. Avoid Common Pitfalls**
- Redundancy: Don't say the same thing multiple ways
- Over-specification: Don't list every possible format, constraint, or edge case
- Defensive bloat: Don't add uncertainty handling or safety warnings unless critical
- Surface patterns: Look for deep semantic strategies, not superficial formatting rules
- Enumerationitis: Avoid long numbered checklists; prefer flowing prose

**3. Craft Effective Instructions**
- State principles clearly and concisely
- Use specific language when precision matters ("identify the missing fact" vs "analyze the information")
- Keep instructions proportional to task complexity
- Test in your mind: would this help on examples you haven't seen?

**4. Preserve What Works**
- If baseline is effective and simple, make minimal changes
- Don't fix what isn't broken
- Complexity should buy you something measurable

The prompt must be a Python dictionary with the following keys:
{module_keys_description}

Output EXACTLY in this format:

```python
{prompt_template}
```
\end{promptbox}

\begin{promptbox}[System Prompt Template for Molecule Optimization]
You are a helpful assistant that proposes a single SMILES string corresponding to a small molecule, based on the task explanation and information about previous molecules that have been proposed. Your proposed molecule should score highly on the {benchmark_name} benchmark.

Task explanation: Maximize -{benchmark_name} - 10 * (1 - QED): larger values mean stronger {benchmark_name} binding and high drug-likeness.
{protein_info_xml}

Your output format should be EXACTLY as follows:

<reasoning>[Chemical hypothesis linking structural changes to expected binding/QED improvements]</reasoning>
<smiles>[Single SMILES string]</smiles>

Generate a small molecule based on these previous examples (sorted in descending order of score):

{examples_text}
\end{promptbox}

\begin{promptbox}[Example of Protein Metadata (ADRB1)]
{ 'target': 'ADRB1', 'accession': 'P08588', 'regions': { 'transmembrane': [[56, 84], [94, 120], [133, 154], [173, 196], [223, 248], [320, 349], [355, 377]], 'extracellular': [[1, 55], [121, 132], [197, 222], [350, 354]], 'cytoplasmic': [[85, 93], [155, 172], [249, 319], [378, 477]], 'disordered': [[269, 307], [403, 477]]}, 'critical_residues': {'mutagenesis': [{'position': [474, 474], 'description': 'Loss of interaction with GOPC.'}, {'position': [474, 474], 'description': 'Loss of interaction with GOPC; when associated with A-477.'}, {'position': [475, 475], 'description': 'Loss of interaction with GOPC. Loss of interaction with RAPGEF2. Abolishes agonist-induced Ras activation.'}, {'position': [475, 475], 'description': 'Loss of interaction with RAPGEF2.'}, {'position': [475, 475], 'description': 'Partial loss of interaction with GOPC.'}, {'position': [476, 476], 'description': 'Partial loss of interaction with GOPC.'}, {'position': [477, 477], 'description': 'Loss of interaction with GOPC.'}, {'position': [477, 477], 'description': 'Loss of interaction with RAPGEF2. Abolishes agonist-induced Ras activation.'}], 'natural_variants': [{'position': [26, 26], 'description': 'in dbSNP:rs34844626'}, {'position': [29, 29], 'description': 'in dbSNP:rs35720093'}, {'position': [31, 31], 'description': 'in dbSNP:rs35230616'}, {'position': [49, 49], 'description': 'correlated with low mean resting heart rate and decreased mortality risk in patients with congestive heart failure; dbSNP:rs1801252'}, {'position': [187, 187], 'description': 'found in individuals with short sleep; results in decreased adenylate cyclase-activating adrenergic receptor signaling; decreased protein stability; dbSNP:rs776439595'}, {'position': [389, 389], 'description': 'increased beta1-adrenergic receptor activity; increased basal activity and increased coupling to heterotrimeric G protein Gs that stimulates the adenylyl cyclase; dbSNP:rs1801253'}, {'position': [399, 399], 'description': 'in dbSNP:rs36052953'}, {'position': [405, 405], 'description': 'in dbSNP:rs35705839'}]}}
\end{promptbox}

\begin{promptbox}[Example of Molecule Metadata (CCCCC)]
valid: 'True'
score: '-1.9121449019886678'
metadata:
  CanonicalSMILES: CCCCC
  InChIKey: OFBQJSOFQDEBGM-UHFFFAOYSA-N
  MolecularFormula: C5H12
  ExactMass: '72.093900384'
  FormalCharge: '0'
  AtomCount: '5'
  HeavyAtomCount: '5'
  HeteroAtomCount: '0'
  BondCount: '4'
  Sp3CarbonFraction: '1.0'
  RingCount: '0'
  AromaticRingCount: '0'
  AliphaticRingCount: '0'
  RotatableBondCount: '2'
  StereoCenterCount: '0'
  MurckoScaffold: ''
  LogP: '2.1965000000000003'
  TopologicalPolarSurfaceArea: '0.0'
  MolarRefractivity: '25.19899999999999'
  HBondDonorCount: '0'
  HBondAcceptorCount: '0'
  BertzComplexityIndex: '7.5097750043269365'
  BalabanJIndex: 2.19060968716425
  HallKierAlpha: '0.0'
  Kappa1: '5.0'
  Chi0v: '4.121320343559642'
  TotalEState: 8.5
  MinEState: 1.34375
  MaxEState: 2.2118055555555554
  PEOE_VSA6: '33.10993926815928'
  SlogP_VSA5: '33.10993926815928'
  BCUTp_1h: '13.744962415414642'
  AccessibleSurfaceArea: '34.19901948541599'
  FunctionalGroups: []
  StructuralAlerts: []
  QuantitativeDrugLikeness: '0.4687855098011332'
  SyntheticAccessibility: '1.699621281696647'
  NaturalProductLikeness: '0.09749981667944'
\end{promptbox}

\end{document}